\documentclass{article}
\usepackage{times}

\usepackage{amsmath,amsfonts,bm}

\def\eqref#1{equation~\ref{#1}}

\def\1{\bm{1}}

\DeclareMathAlphabet{\mathsfit}{\encodingdefault}{\sfdefault}{m}{sl}
\SetMathAlphabet{\mathsfit}{bold}{\encodingdefault}{\sfdefault}{bx}{n}

\newcommand{\Conv}{\mathop{\scalebox{1.7}{\raisebox{-0.2ex}{$\ast$}}}}%

\usepackage[utf8]{inputenc} %
\usepackage[T1]{fontenc}    %
\usepackage{hyperref}       %
\usepackage{url}            %
\usepackage{booktabs}       %
\usepackage{amsfonts}       %
\usepackage{nicefrac}       %
\usepackage{microtype}      %
\usepackage{xcolor}         %

\usepackage{amsmath}
\usepackage{amssymb}
\usepackage{amsthm}
\usepackage{bm}
\usepackage{graphicx}
\usepackage{mathtools}
\usepackage{subfigure}
\usepackage[capitalize,noabbrev]{cleveref}
\usepackage{textcomp} %
\usepackage{minitoc}
\usepackage{authblk}

\usepackage{placeins}

\usepackage[numbers]{natbib}
\theoremstyle{plain}
\newtheorem{theorem}{Theorem}[section]
\newtheorem{proposition}[theorem]{Proposition}

\newtheorem{corollary}[theorem]{Corollary}
\theoremstyle{definition}
\newtheorem{definition}[theorem]{Definition}

\theoremstyle{remark}

\title{LongSSM: On the Length Extension of State-space Models in Language Modelling}
\author[1]{Shida Wang}
\affil[1]{Department of Mathematics, National University of Singapore}

\begin{document}

\maketitle

\begin{abstract}
In this paper, we investigate the length-extension of state-space models (SSMs) in language modeling. 
Length extension involves training models on short sequences and testing them on longer ones.
We show that state-space models trained with zero hidden states initialization have difficulty doing length extension. 
We explain this difficulty by pointing out the length extension is equivalent to polynomial extrapolation. 
Based on the theory, we propose a simple yet effective method - changing the hidden states initialization scheme - to improve the length extension. 
Moreover, our method shows that using long training sequence length is beneficial but not necessary to length extension. 
Changing the hidden state initialization enables the efficient training of long-memory model with a smaller training context length. 
\end{abstract}

\section{Introduction}
    
Large language models~\citep{brown2020.LanguageModelsArea} are usually trained on large corpus with a fixed context length (e.g., 2048 tokens).
However, attention-based transformer~\citep{brown2020.LanguageModelsArea} has an $O(T^2)$ asymptotic growth with respect to the sequence length. 
The cost for training and inference is even higher when we are working with long sequences. 
Recently, state-space models~\citep{gu2020.HiPPORecurrentMemorya,gu2021.EfficientlyModelingLong,gu2023.MambaLinearTimeSequence,de2024.GriffinMixingGated} and linear-attention-based transformers~\citep{katharopoulos2020.TransformersAreRNNs,sun2023retentive,yang2023.GatedLinearAttention} have shown the potential to replace the attention-based transformers~\citep{brown2020.LanguageModelsArea}.
SSMs are recurrent models characterized by parallelism in sequence length and inference cost that remains independent of length.

Despite state-space models having a recurrent form and thereby inducing an \textbf{``infinite-in-time'' memory} of the input history, they tend to exhibit limited length extension beyond the training sequence length in mamba~\citep{gu2023.MambaLinearTimeSequence}.
In practical applications, where the target inference context often exceeds the length of the training sequence and can even be infinite, a pertinent question arises: 
\emph{
Is it possible to train a model with the ability to extend its memory beyond the constraints of a finite training sequence length?
}
The assumption of a finite training sequence length is both reasonable and necessary, given the constraints of GPU memory and the comparatively short training length, especially when compared with the infinite inference length found in real-world applications.

In the earliest transformer model~\citep{vaswani2017.AttentionAllYou}, achieving length extension is challenging, usually constrained by the limitations of absolute position encoding~\citep{vaswani2017.AttentionAllYou, press2022.TrainShortTest}. 
\citet{press2022.TrainShortTest} have demonstrated that introducing attention with linear bias serves as an effective solution to address this limitation and enable length extension.
Apart from the additive bias, another stream of works is constructing relative position embedding~\citep{su2022.RoFormerEnhancedTransformer,sun2022.LengthExtrapolatableTransformer,chen2023.CLEXContinuousLength}. 
In this paper, we adopt the backpropagation through time method that is orthogonal to these previous approaches, and can be used to improve state-space models' length extension capability. 
Moreover, our method shows that the length extension capability can be achieved without using a long training sequence~(\cref{fig:comparison_of_two_hidden_states_initailizations}). 

We summarize our main contributions as follow:
\begin{enumerate}
    \item We show why the zero hidden states initialization scheme has difficulty doing length extension.
    \item Based on the difficulty for zero-initialization case, we introduce the training approach that leverages previous hidden states with no batch-level shuffling.
    \item We show the length extension can be achieved without a long training sequence length. In particular, we show the feasibility to train a model with \textbf{training sequence length 16} and \textbf{truncated BPTT}, but has \textbf{length extension up to 32768}. 
\end{enumerate}

\begin{table*}[tbh!]
    \caption{Comparison of asymptotic training/inference step cost for attention-based transformers~\citep{brown2020.LanguageModelsArea} and state-space models with respect to context length $T$.
    }
    \label{table:comparison_of_train_inf_cost}
    \centering
    \begin{tabular}{c|cc}
    \toprule
                      & Attention-based transformer & State-space models/Linear-attention \\
    \midrule
    Training cost& $O(T^2)$      & $O(T)$ \\
    Inference cost    & $O(T^2)$      & $O(1)$ \\
    \bottomrule
    \end{tabular}
\end{table*}    

\paragraph{Notation}
We use the bold face to represent the sequence while then normal letters are scalars, vectors or functions. 
We use $\|\cdot\|$ to denote norms over sequences of vectors, or functions, while $|\cdot|$ (with subscripts) represents the norm of number, vector or weights tuple.
Here $|x|_\infty := \max_{i} |x_i|, |x|_2 := \sqrt{\sum_{i} x_i^2}, |x|_1 := \sum_{i} |x_i|$ are the usual max ($L_{\infty}$) norm, $L_2$ norm and $L_1$ norm. 
Let $m$ be the hidden dimension and $d$ be the input dimension.

\section{Background}

In this section, we first introduce the state-space models (SSMs).
Compared with traditional nonlinear RNNs, they have better parallelism across sequence length in the sense that fast Fourier transform and associative scan can be used to reduce the training latency. 
Next, we give the definition of three types of length extension capability. 
The aim of this paper is not to improve the length extension towards a particular length but to achieve the monotonic perplexity decrease for weak length extension. 

\subsection{State-space models}

State-space models~\citep{gu2021.EfficientlyModelingLong} have layer-wise nonlinear activations while the traditional nonlinear RNNs have recurrent nonlinear activations (see the comparison of SSMs and RNNs in \cref{sec:comparison_ssm_rnn}). 
\begin{align}
    h_{k+1}   & = Wh_k + (Ux_k + b), \quad h_0 = 0 \in \mathbb{R}^m\\
    \hat{y}_k & = C \bm{\sigma}(h_k), \quad 1 \leq k \leq T.
\end{align}
Here $h_k \in \mathbb{R}^m, W \in \mathbb{R}^{m \times m}, U \in \mathbb{R}^{m \times d}, b \in \mathbb{R}^{m}, C \in \mathbb{R}^{d \times m}.$
The corresponding continuous-time form is 
\begin{align}
    \frac{dh_t}{dt} = Wh_t + (Ux_t + b), \quad \hat{y}_t = C \bm{\sigma}(h_t). 
\end{align}
The solution of $h_t$ in convolution form is $h_t = h_0 + \int_{0}^{t} e^{W(t-s)} (Ux_s + b) ds$.

\paragraph{The convolution form of SSMs} 
Based on the above continuous-time formulation, hidden states sequences can be written into the following convolution form 
\begin{equation}
\label{eq:hidden_state_convolution}
    \mathbf{h} = \rho(t) \Conv (U\mathbf{x}+b) 
\end{equation}
The convolution kernel is $\rho(t) = e^{Wt}$.
Given the convolution form in \cref{eq:hidden_state_convolution}, FFT~\citep{gu2021.EfficientlyModelingLong} can be used to accelerate the computation. 
Compared with $O(T^2)$ cost of attention matrix, it only takes $O(T \log T)$ to evaluate the hidden states $\mathbf{h}$ and corresponding outputs $\hat{\mathbf{y}}$. 

\paragraph{Scan-based acceleration for models with input-dependent gating}
Recent advancements in state-space models have significantly enhanced their expressiveness and approximation capabilities through the incorporation of input-dependent gating mechanisms.
The input-dependent gating refers to the generalization of $W$ and $Ux_k + b$ to $W(x_k) \in \mathbb{R}^{m \times d}$ and $U(x_k) \in \mathbb{R}^{m \times d}$. 
\begin{align}
    h_{k+1}   & = W(x_k) \odot h_k + U(x_k), \quad h_0 = 0 \in \mathbb{R}^{m \times d} \\
    \hat{y}_k & = C \bm{\sigma}(h_k), \quad 1 \leq k \leq T.
\end{align}
Here $\odot$ is the element-wise product. 

As input-dependent gating disrupts the convolution structure and negates the speed benefits derived from FFT, it is still feasible to employ scan-based acceleration techniques~\citep{martin2018.ParallelizingLinearRecurrent}, achieving the $O(T)$ training cost. 
In \cref{subsec:associativity}, we show the associativity of the following binary operator $\circ$ defined over tuple $(W, h)$:
\begin{align}
    (W_1, h_1) \circ (W_2, h_2) = (W_2 \odot W_1, h_1 + W_1 \odot h_2).
\end{align}
The initialization is $h_0 = 0$ and hidden states $h_k$ can be achieved from
\begin{equation}
    (\textrm{\_}, h_k) = (W(x_k), U(x_k)) \circ \cdots \circ (W(x_1), U(x_1)) \circ (I, 0).
\end{equation}
If we embed $U(x_k)$ in $\mathbb{R}^{m \times m}$ rather than $\mathbb{R}^{m \times d}$, this corresponds to the gated linear attention~\citep{yang2023.GatedLinearAttention} whose hidden states are 2D square matrices. 
We summarize the differences in \cref{table:diffs_s5_mamba_gla} of \cref{subsec:associativity}.

\subsection{Length extension}

Length extension has been widely studied for transformers~\citep{press2022.TrainShortTest, su2022.RoFormerEnhancedTransformer, sun2022.LengthExtrapolatableTransformer, chen2023.CLEXContinuousLength}. 
This is an essential attribute for models designed for infinite contexts windows (writing novels~\citep{yuan2022.WordcraftStoryWriting}, autonomous driving~\citep{chen2023.EndtoendAutonomousDriving}, online learning~\citep{marschall2020.UnifiedFrameworkOnline}). %
However, it is shown that the state-of-the-art Mamba~\citep{gu2023.MambaLinearTimeSequence} failed to achieve length extension beyond 4k\footnote{\url{https://openreview.net/forum?id=AL1fq05o7H}, \url{https://imgbb.com/XVT0hGJ}}. 

Building upon existing research in length extension, we initially establish specific concepts to qualitatively classify models based on their capability to extend length. 

\begin{figure*}[tbh!]{
    \centering
    \includegraphics[width=0.8\textwidth]{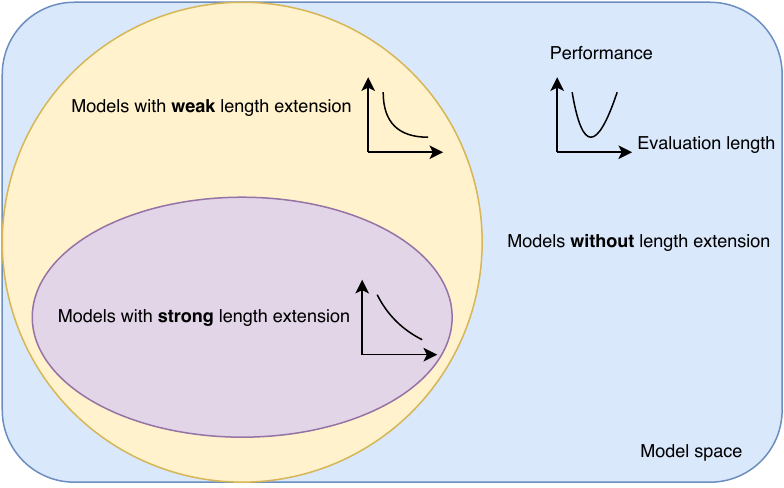}
    \caption{Three types of length-extension capabilities.
    }
    \label{fig:three_context_extension}
}
\end{figure*}

\begin{definition}
For auto-regressive language modeling, the entropy $H(p) = -\sum_{x} p(x) \log p(x)$ of the target language $p$ is fixed.
Here we define three types of length-extension capability based on the monotonicity of the perplexity: 
\begin{enumerate}
    \item{\textbf{(Strong length extension)}}: For some $T_0 > 0$, $\forall T > T_0, \textrm{perplexity}_{T+1} < \textrm{perplexity}_{T}$.
    \item{\textbf{(Weak length extension)}}: For some $T_0 > 0$, $\forall T > T_0, \textrm{perplexity}_{T+1} \leq \textrm{perplexity}_{T}$.
    \item{\textbf{(No length extension)}}: If there does not exists $T_0$ such that weak length extension holds. 
\end{enumerate}
\end{definition}

As demonstrated in \cref{fig:three_context_extension}, models with strong length extension are a subset of those with weak length extension. 

In \cref{fig:length_extension_difficulty_in_mamba}, we evaluate the length extension difficulty for Mamba across different model sizes. 
Mamba is trained with sequence length $T=2048$ and has difficulty maintaining the small perplexity beyond length $T \geq 4096$. 

\begin{figure}[tbh!]{
    \centering
    \includegraphics[width=0.6\textwidth]{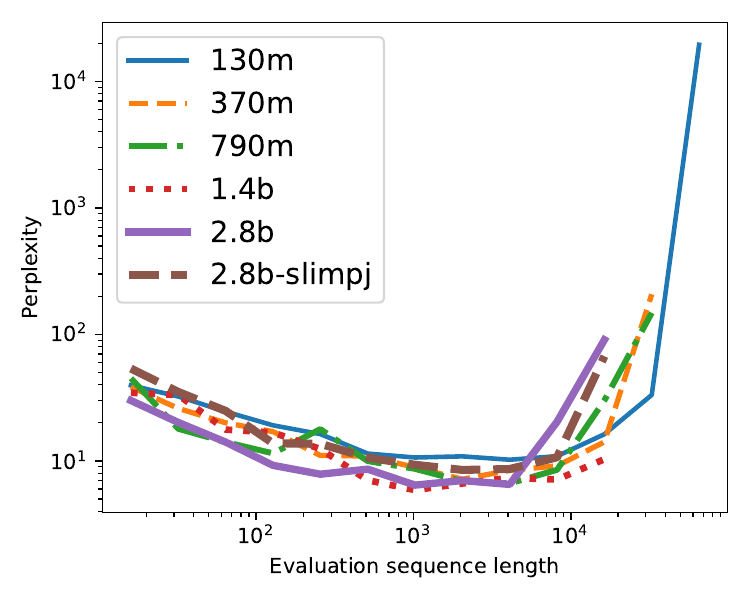
    }
    \caption{
        Length extension performance of Mamba evaluated over the Pile dataset~\citep{gao2020.Pile800GBDataset}. 
        The models are trained with a sequence length of 2048. 
        Although perplexity remains finite for sequences up to 4096, it increases significantly for lengths beyond 8192.
    }
    \label{fig:length_extension_difficulty_in_mamba}
}
\end{figure}
    
Based on the above definitions, a natural question arises: 
\emph{Can length extension exist for autoregressive language modeling?} 
We demonstrate that if a language is viewed as a shift-invariant sequence of random variables, then the weak length extension exists. 

\begin{theorem}[Existence of weak length extension in autoregressive language modeling]
    Assume the entropies of language across different sequence lengths are all finite. 
    Consider the autoregressive language modeling as the learning of sequence of random variables $\{X_k\}_{k=1}^\infty$. 
    The ideal autoregressive language models return the next random variable $\mathbf{X}_{T+1}$ based on the previous random variables $X_{[0, \dots, T]}$. 
    \begin{equation}
        \mathbf{Model}((X_1, \dots, X_T)) = X_{T+1}. 
    \end{equation}

    Consider the entropy of this autoregressive language model
    \begin{align}
        H((X_1, \dots, X_T)) 
        & = -\sum_{x_i \in X_i} p((x_1, \dots, x_T)) \log p((x_1, \dots, x_T)) \\
        & = \sum_{i=1}^T H(X_i | X_1, \dots, X_{i-1}). 
    \end{align}
    By monotonicity of entropy and shift-invariant property, we know $H(X_i | X_1, \dots, X_{i-1}) \leq H(X_i | X_2, \dots, X_{i-1}) = H(X_{i-1} | X_1, \dots, X_{i-2})$.
    By the boundedness of $H$ we know $\lim_{i \to \infty} H(X_i | X_1, \dots, X_{i-1})= 0$.
\end{theorem}
See the proof based on the information theory in \cref{subsec:existence_of_weak_length_extension}.
By the universal approximation property of recurrent state-space models \citep{wang2023.StatespaceModelsLayerwisea}, we know the target autoregressive next-word prediction sequence $ \bigg (X_1, (X_2 | X_1), \dots, (X_T | X_1, \dots, X_{T-1}) \bigg )$ can be approximated by the recurrent model $\bigg ( \mathbf{Model}(\emptyset), \mathbf{Model}(X_1), \dots, \mathbf{Model}(X_1, \dots, X_{T-1}) \bigg )$. 
Therefore the entropy of sequence model $\lim_{T \to \infty} H(\mathbf{Model}(X_1, \dots, X_{T-1})) = \lim_{T \to \infty} H(X_T | X_1, \dots, X_{T-1})$ is also decaying to 0.

\section{Main results}

In this section, we present a theoretical analysis of the challenges associated with length extension in SSMs that have zero-initialized hidden states, as detailed in \cref{subsec:length_extension_is_extrapolation}.
We illustrate that doing well in length extension is analogous to performing well in polynomial extrapolation.
In \cref{subsec:convert_the_extrapolation_to_interpolation}, we argue that setting proper initialization for hidden states can transform the extrapolation challenge into an interpolation problem, thereby improving length extension performance.

\subsection{Length extension is extrapolation}
\label{subsec:length_extension_is_extrapolation}

In approximation theory, state-space models are universal approximators for bounded causal continuous time-homogeneous regular nonlinear functionals, as detailed by \citet{wang2023.StatespaceModelsLayerwisea}. 
This outcome ensures the existence of a suitable model capable of learning the sequence-to-sequence relationships over unbounded time horizon $(-\infty, t)$ with any desirable tolerance.
In practice, models are trained with a fixed finite length $T$. 
Therefore, it becomes crucial to assess whether such finite-window-trained models can effectively capture long-term memory beyond their training scope.
In this context, we explore length extension in a simplified linear framework, which can be similarly extended to multi-layer nonlinear SSMs.

Consider the learning of linear functionals~\citep{li2020.CurseMemoryRecurrent,jiang2023.BriefSurveyApproximationa} by single-layer state-space model, this linear functional target comes with a unique representation: 
$y_t = \mathbf{H}_t(\mathbf{x}) = \int_{0}^\infty \rho_{s} x_{t-s} ds$ with $|\rho|_1 := \int_{0}^\infty |\rho_{s}| ds < \infty$
while the single-layer state-space model (without layerwise activation) can be represented by 
$\hat{y}_t = \widehat{\mathbf{H}}_t(\mathbf{x}) = \int_{0}^T C e^{Ws} U x_{t-s} ds + \hat{y}_0.$
The learning of target $\mathbf{H}$ by model $\widehat{\mathbf{H}}$ is equivalent to approximating the memory function $\rho(t): [0, \infty) \to \mathbb{R}$ with the SSM memory kernel $\hat{\rho}(t) = C e^{Wt} U$. 
Consider the following error decomposition
\begin{align*}
    & |y_T - \hat{y}_T| = \left | \int_{0}^\infty \rho_s x_{T-s} ds - \left ( \int_{0}^T \hat{\rho}_s x_{T-s} ds + \hat{y}_0 \right ) \right | \\
    & \leq \left | \int_{T}^\infty \rho_s x_{T-s} ds - \hat{y}_0 \right | + \left | \int_{0}^T \rho_s x_{T-s} ds - \int_{0}^T \hat{\rho}^*_s x_{T-s} ds \right | + \left | \int_{0}^T \hat{\rho}^*_s x_{T-s} ds - \int_{0}^T \hat{\rho}_s x_{T-s} ds \right |.
\end{align*}

Here $\hat{\rho}^*$ is the ``optimal'' model memory function while $\hat{\rho}$ is the achieved model memory function. 
The three terms in the error decomposition correspond to the \textbf{length extension error}, \textbf{finite time approximation error}, \textbf{optimization error}. 
For any fixed target $\mathbf{H}$, as the hidden dimension $m$ increases, the finite time approximation error decays to 0.
Given sufficient data and abundant computational resources, the optimization error decreases to zero through gradient-based optimization. 
However, the length extension error cannot be reduced by simply increasing hidden dimension or improve the training over finite context data.
    
During the inference, the error decomposition for $t>T$ is 
\begin{align}
    |y_t - \hat{y}_t| & \leq \left |\int_t^{\infty} \rho_{s} x_{t-s} ds - \hat{y}_0 \right | + \left |\int_T^t (\rho_{s} - \hat{\rho}_{s}) x_{t-s} ds \right | + \left |\int_0^T (\rho_{s} - \hat{\rho}_{s}) x_{t-s} ds \right |. 
\end{align}
With the first error unobserved and third error minimized in training, the major error for length extrapolation is the second term which be bounded by the form of $\int_T^t |\rho_{s} - \hat{\rho}_{s}| ds$.
By change of variable $u = e^{-s}$, take $\mathcal{T} \rho_u = \rho_{- \log u}, u \in (0, 1]$. 
\begin{equation}
    \int_T^t |\rho_{s} - \hat{\rho}_{s}| ds = \int_{e^{-t}}^{e^{-T}} \left | \mathcal{T}\rho_{u} - \sum_{k=1}^m c_i u^\lambda_i \right | \frac{1}{u} du. 
\end{equation}
The error between $[T, t]$ is equivalent to evaluate the polynomial extrapolation error of $\frac{\mathcal{T}\rho_{u}}{u}$ over $u \in [e^{-t}, e^{-T}]$. 
This is said to be polynomial extrapolation as the coefficient of the polynomials are only fitted over interval $u \in [e^{-T}, 1]$. 
As the models are usually overparameterized, the minimizer of truncated loss $E_{[0, T]}$ is not the global minimizer for $E_{[0, \infty)}$. 
In \cref{subsec:overfit_in_length_extension}, we further show the similarity between nonlinear state-space model length extension and polynomial extrapolation. In particular, the overfitting phenomenon gets worse as the number of parmaeters increased. 

\begin{figure*}[ht!]{
    \centering
    \includegraphics[width=0.75\textwidth]{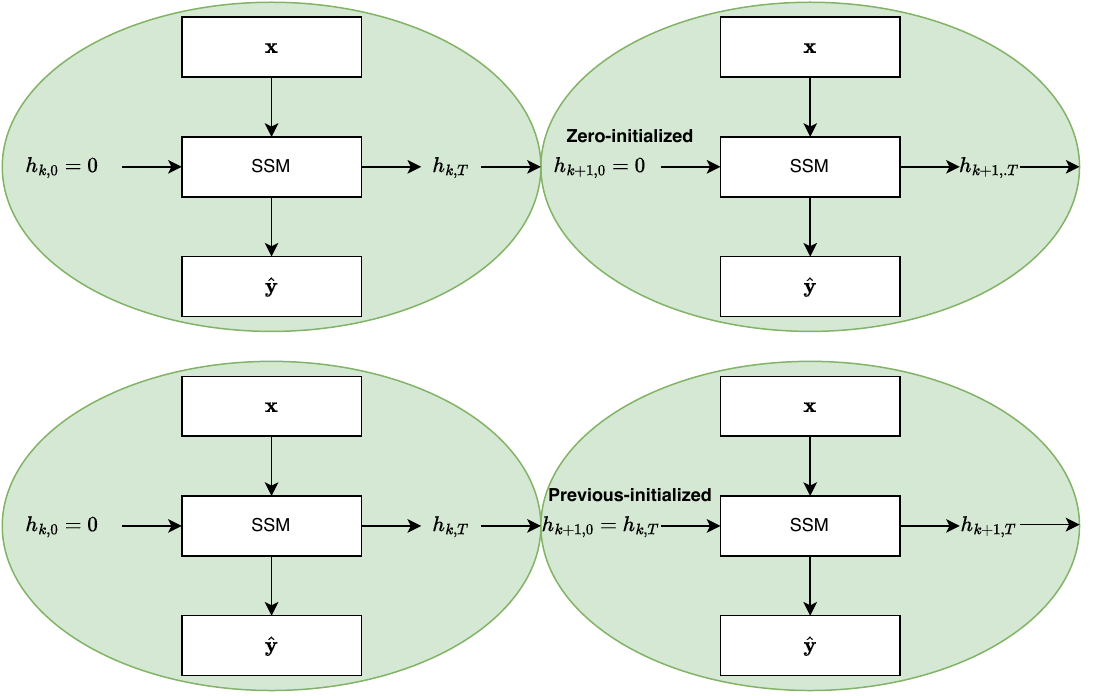}
    \caption{Graphical demonstration of the difference between zero-initialized hidden states and previous-initialized hidden states (truncated backpropagation through time) in training. 
    }
    \label{fig:method}
}
\end{figure*}

\begin{figure*}[t!]{
    \centering
    \subfigure[][Validation loss]{\includegraphics[width=0.47\textwidth]{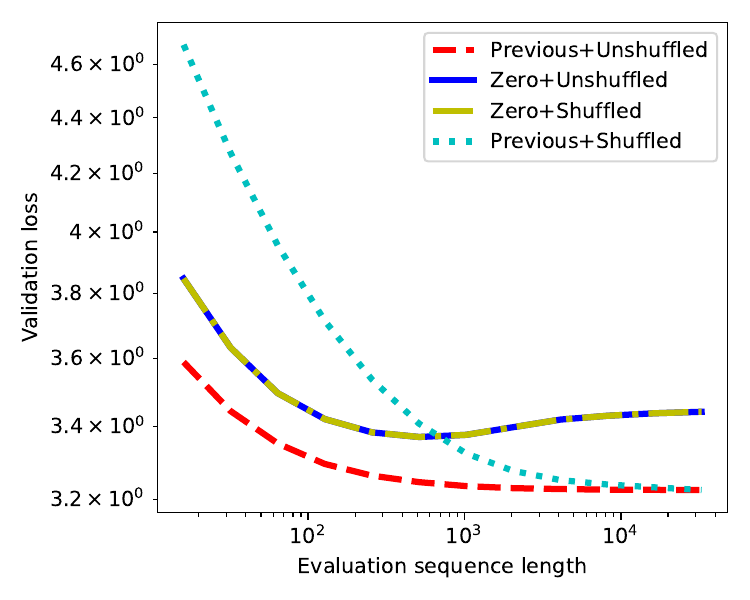}
    }
    \subfigure[][Test loss]{\includegraphics[width=0.47\textwidth]{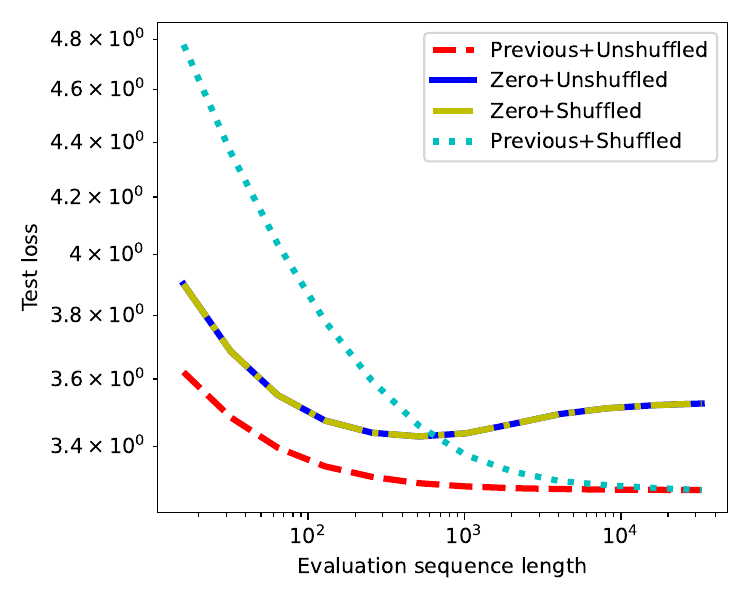}
    }
    \caption{
        Comparison of two hidden states initialization methods over 6-layer Mamba with 30M parameters. 
        Both the zero-initialized and previous-initialized models are trained over training sequence length $T=32$.
        The zero-initialized model has difficulty extrapolating beyond 1024 while the the previous-initialized model has length extrapolation up to $T=32768$. 
        While the previous hidden state methods achieve the length extension over unshuffled test dataset, when the data is shuffled, models trained with previous hidden state also suffer from the noisy information in the hidden states. 
    }
    \label{fig:comparison_of_two_hidden_states_initailizations}
}
\end{figure*}

\subsection{Convert the extrapolation to interpolation}
\label{subsec:convert_the_extrapolation_to_interpolation}

In the training of state-space models, the hidden states are usually zero-initialized between different batches. 
As shown in \cref{fig:method}, we set the initialization of hidden states from the previous batch $h_{k+1, 0}=h_{k, T}$ instead of zeros $h_{k,0}=0$. 
This corresponds to the truncated backpropagation through time method~\citep{jaeger2002.TutorialTrainingRecurrent}.
This change of hidden states initialization requires the dataloader to load consecutive text instead of shuffling them in the batch level. 
We compare the effects of data shuffling on the following two initialization schemes in \cref{fig:comparison_of_two_hidden_states_initailizations}.

The zero-initialized model gives almost the same loss curves over the shuffled dataset and unshuffled dataset. 
In contrast, the model trained with previous-initialized hidden sates have smaller validation/test loss and monotonically decreasing loss in the length extension sense. 
As the previous-initialized model suffer when the evaluation dataset is shuffled dataset, it indicates that the model does extract the information from the non-zero previous hidden states.

\section{Numerical results}

In this section, we first provide the numerical evidence that training with longer context is generally better but not necessary for length extension~(\cref{subsec:longer_training_context_is_beneficial_but_not_necessary_for_length_extension}). 
Then, we further demonstrate that with previous-initialized hidden states, the models can achieve even better length-extension performance that general proper-trained zero-initialized models~(\cref{subsec:previous_initialized_hidden_states_improve_the_length_extension_capability}). 
The disadvantage of this previous-initialized training is discussed in \cref{subsec:on_the_disadvantages_of_previous}.

\begin{figure*}[tbh!]{
    \centering
    \subfigure[][S5]{\includegraphics[width=0.47\textwidth]{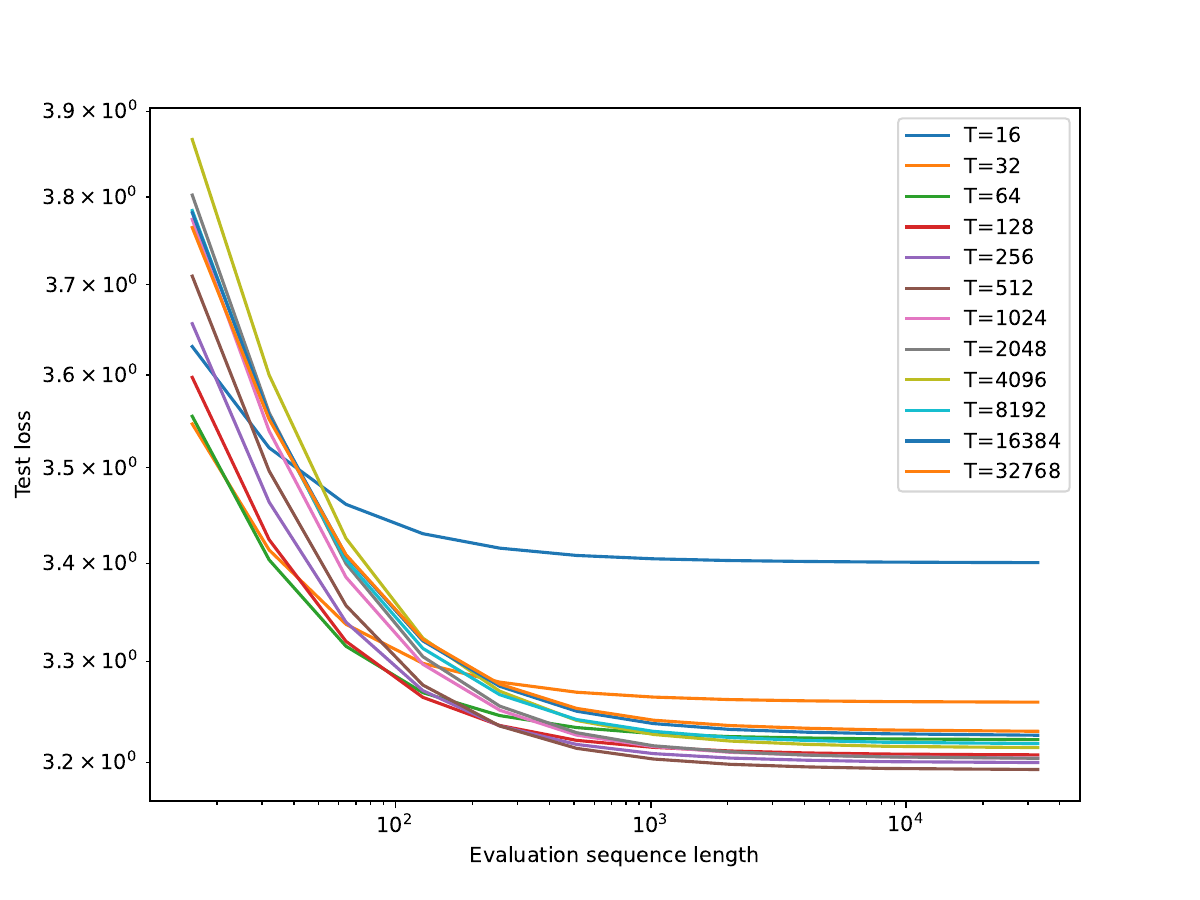}
    }
    \subfigure[][Mamba]{\includegraphics[width=0.47\textwidth]{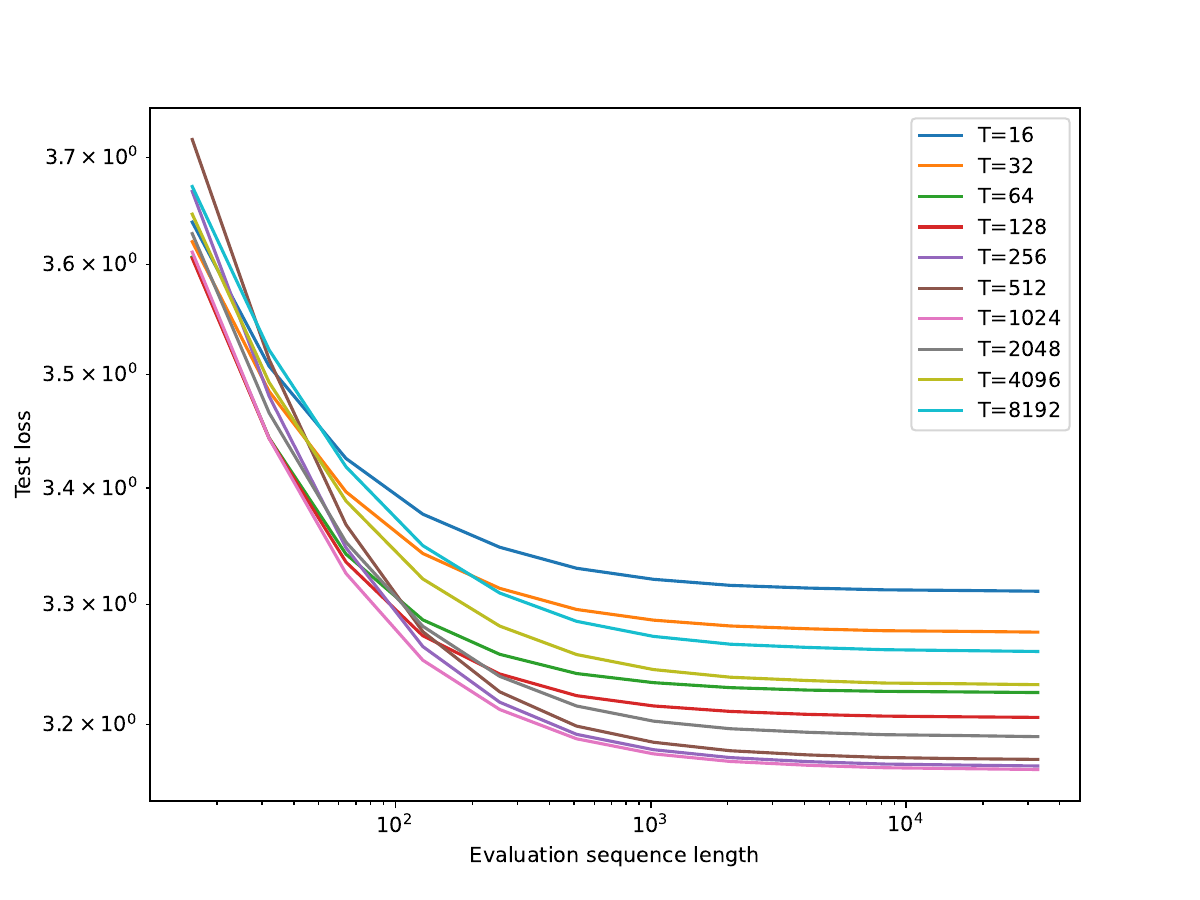}
    }
    \caption{Length extension of models trained with different sequence length using previous-initialized hidden states. 
        We train 6-layer S5~\citep{smith2023.SimplifiedStateSpace} up to training length $T=32768$ and 6-layer Mamba~\citep{gu2023.MambaLinearTimeSequence} up to training length $T=8192$. 
        Mamba has a larger hidden states dimension therefore the maximum training length is smaller (on the same GPU). 
        It can be seen that training with sequence length $T=1024$ is slightly better than shorter/longer sequence length. 
    }
    \label{fig:S5_Mamba_longer_is_generally_better}
    }
\end{figure*}

\subsection{Longer training context is \textbf{beneficial but not necessary} for length extension}
\label{subsec:longer_training_context_is_beneficial_but_not_necessary_for_length_extension}

We show in \cref{fig:S5_Mamba_longer_is_generally_better} that since inheriting the previous hidden states are approximating the gradient with longer training sequence length, the performance of models trained over longer sequence generally have better length extension capability. 
Since the models with previous-initialized hidden states have monotonically decreasing perplexity, therefore the length extension property can be achieved without long training sequence length.

\begin{figure*}[ht!]{
    \centering
    \includegraphics[width=0.63\textwidth]{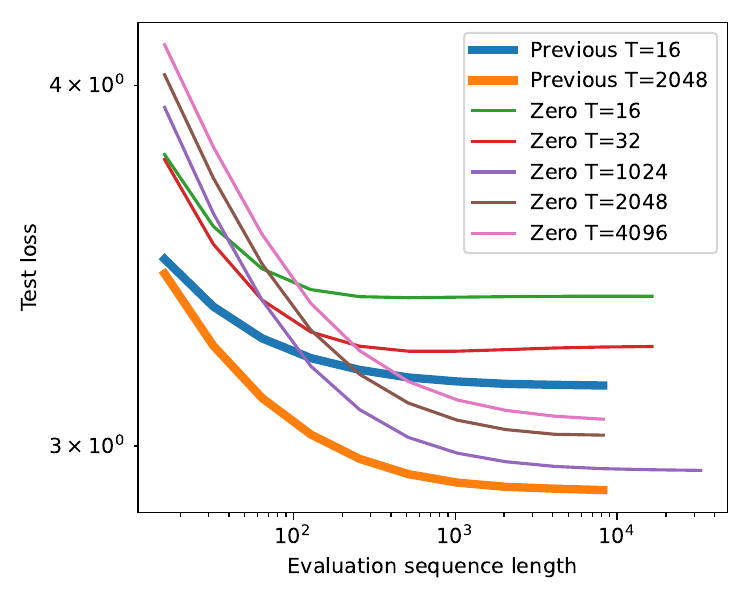}
    \caption{
        We change the training sequence length and show that, despite zero-initialized models showing adequate length extension capabilities, models trained with previous hidden states consistently surpass them across all evaluated sequence lengths. 
        Throughout our experiments, Mamba models with 180M parameters trained on the Wikitext103 dataset maintain consistent training settings.
    }
    \label{fig:benfit_from-longer_sequences}
}
\end{figure*}

\subsection{Previous initialized hidden states improve the length extension capability}
\label{subsec:previous_initialized_hidden_states_improve_the_length_extension_capability}

In \cref{fig:benfit_from-longer_sequences}, we present the length extension curves for the 180M Mamba model, trained using various sequence lengths and different schemes for (training) hidden states initializations.
The model with previous initialization, when trained on sequences of length $T=16$, outperforms the zero-initialized model trained on both $T=16$ and $T=32$ sequence lengths.
Furthermore, the model trained with a sequence length of $T=2048$ demonstrates superior length extension performance across both short and long sequences compared to all models with zero initialization.
All these models are trained in the same hyperparameter setting.

\begin{figure*}[ht!]{
    \centering
    \subfigure[][Zero-initialized training]{\includegraphics[width=0.47\textwidth]{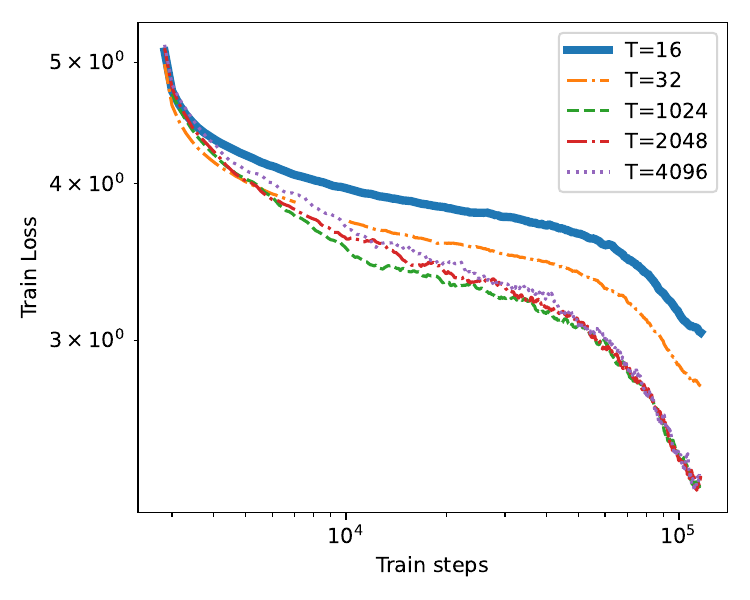}}
    \subfigure[][Previous-initialized training]{\includegraphics[width=0.47\textwidth]{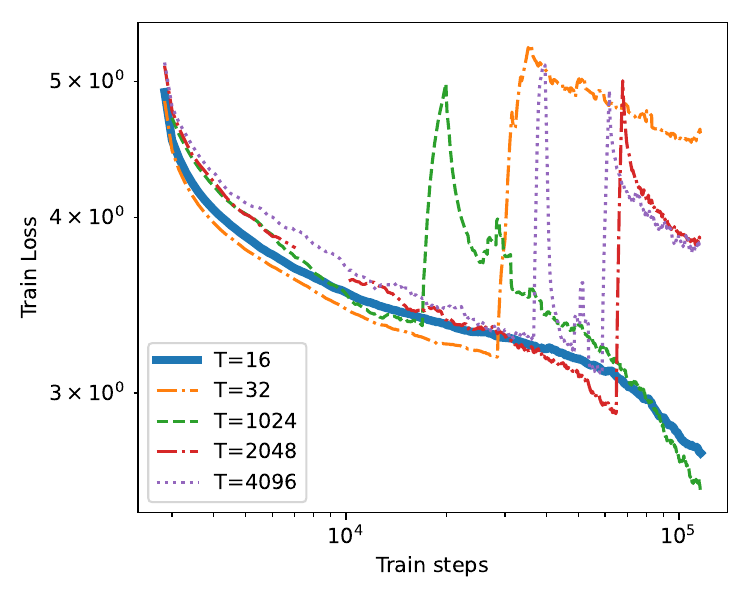}}
    \caption{
        Under the same training setting (learning rate is 0.001), we show that the training of 140M previous-initialized S5~\citep{smith2023.SimplifiedStateSpace} come with severe \textbf{training instability}. 
    }
    \label{fig:140M_S5_stability}
    }
\end{figure*}

\subsection{On the disadvantages of previous-initialized training}
\label{subsec:on_the_disadvantages_of_previous}

In the previous subsections, we explore the advantages of length extension through modifications in hidden states initialization during training. 
Here we evaluate the efficacy of the previously-initialized (truncated-BPTT) method and uncover the significant challenges it presents in terms of training stability.
As illustrated in \cref{fig:140M_S5_stability}, particularly with relatively large models, training the 140M S5 model with previously-initialized hidden states (on the right) exhibits notable instability compared to zero-initialized training (on the left).
As training setting are the same, this result shows the instability of current previous-initialized training. 
In \cref{subsec:sensitivity_of_recurrent_weights}, we further examine stability through the lens of both hidden state bounds and weight precision in the setting of long-term memory learning.
Future direction includes the study of achieving length extension in a more stable manner.

\section{Related works}

Recurrent neural networks~\citep{rumelhart1986.LearningRepresentationsBackpropagating} are widely used in sequence modeling. 
Variants such as LSTM~\citep{hochreiter1997.LongShorttermMemory} and GRU~\citep{cho2014.LearningPhraseRepresentations} are efficient to model sequence-to-sequence relationship but they suffer from problems such as vanishing/exploding gradient~\citep{bengio1994.LearningLongtermDependencies, hochreiter1998.VanishingGradientProblem} and exponentially decaying memory~\citep{jiang2023.BriefSurveyApproximationa,wang2023.InverseApproximationTheory,wang2023.StatespaceModelsLayerwisea}. 
As nonlinear RNNs cannot be parallelized in time, back propagation through time (BPTT) \citep{jaeger2002.TutorialTrainingRecurrent} is widely used to speed up the training of long sequences. 
State-space models~\citep{gu2021.EfficientlyModelingLong,wang2023.StableSSMAlleviatingCurse} relax the training difficulty as the linear RNN layer can be parallelized (in time) via FFT or associative scan~\citep{martin2018.ParallelizingLinearRecurrent}.
Mamba~\citep{gu2023.MambaLinearTimeSequence} shows that recurrent models without recurrent nonlinearity can have matching performance against transformer over many tasks while maintaining a low inference cost. 

For transformers, Rotary Position Embedding (RoPE)~\citep{su2022.RoFormerEnhancedTransformer} integrates relative positional information into the attention matrix but still cannot achieve reasonable performance beyond the pretrained length.
Position Interpolation (PI)~\citep{chen2023.ExtendingContextWindow} introduces a linear rescaling in RoPE and achieves the extension from 2048 to 32768.
In \citet{chen2023.CLEXContinuousLength}, they introduce a trainable neural ODE~\citep{chen2019.NeuralOrdinaryDifferential} into the position encoding, enabling more fine-grained long context extension.
Additive bias~\citep{raffel2020.ExploringLimitsTransfer} is another approach to achieve the length extension. 
ALiBi~\citep{press2022.TrainShortTest, al-khateeb2023.PositionInterpolationImproves} is the first effective method to do length extensions, it has been shown to have monotonically decreasing perplexity up to length 3072 for models trained over 64. 

It is well known that polynomial extrapolation are ill-conditioned\citep{demanet2019.StableExtrapolationAnalytic} and global minimizers of under-determined system are not unique. 
Empirical evidence~\citep{chen2023.ExtendingContextWindow} shows the difficulty of extrapolation in the sense that almost every learned curve has the extrapolation issue.

\section{Conclusion}

In this paper, we investigate the length extension problem in language modeling, particularly focusing on state-space models. 
We emphasize the challenge faced by zero-initialized SSMs in achieving length extension, which essentially boils down to a problem of polynomial extrapolation.
Building upon the above observation, we adopt a simple yet effective hidden states initialization scheme during training.
This method significantly enhances the model's performance on longer contexts without compromising its effectiveness on shorter ones. 
A model with training length $T=16$ can extend to $T=32K$, showcasing a consistent decrease in perplexity, as illustrated in \cref{fig:comparison_of_two_hidden_states_initailizations}. 
Contrary to the common believe that backpropagation is restricted to training lengths of 10-20x~\citep{jaeger2002.TutorialTrainingRecurrent}, our approach is beneficial when the primary goal is length extension, leading to a dramatic reduction in GPU memory requirements—by up to 2000 times (from 32768 to 16).
This discovery suggests that training state-space models with \textbf{longer training contexts is desirable but not necessary} for achieving effective length extension.

\newpage

\section*{Acknowledgements}

Shida Wang is supported by NUS-RMI Scholarship. 
We thank Qian Liu, Tianbo Li, Chao Du, Min Lin for helpful discussions.

\bibliography{example_paper}

\begin{thebibliography}{36}
\providecommand{\natexlab}[1]{#1}
\providecommand{\url}[1]{\texttt{#1}}
\expandafter\ifx\csname urlstyle\endcsname\relax
  \providecommand{\doi}[1]{doi: #1}\else
  \providecommand{\doi}{doi: \begingroup \urlstyle{rm}\Url}\fi

\bibitem[Brown et~al.(2020)Brown, Mann, Ryder, Subbiah, Kaplan, Dhariwal, Neelakantan, Shyam, Sastry, and Askell]{brown2020.LanguageModelsArea}
Tom Brown, Benjamin Mann, Nick Ryder, Melanie Subbiah, Jared~D. Kaplan, Prafulla Dhariwal, Arvind Neelakantan, Pranav Shyam, Girish Sastry, and Amanda Askell.
\newblock Language models are few-shot learners.
\newblock \emph{Advances in neural information processing systems}, 33:\penalty0 1877--1901, 2020.

\bibitem[Gu et~al.(2020)Gu, Dao, Ermon, Rudra, and R{\'e}]{gu2020.HiPPORecurrentMemorya}
Albert Gu, Tri Dao, Stefano Ermon, Atri Rudra, and Christopher R{\'e}.
\newblock {{HiPPO}}: {{Recurrent Memory}} with {{Optimal Polynomial Projections}}.
\newblock In \emph{Advances in {{Neural Information Processing Systems}}}, volume~33, pages 1474--1487. Curran Associates, Inc., 2020.

\bibitem[Gu et~al.(2021)Gu, Goel, and Re]{gu2021.EfficientlyModelingLong}
Albert Gu, Karan Goel, and Christopher Re.
\newblock Efficiently {{Modeling Long Sequences}} with {{Structured State Spaces}}.
\newblock In \emph{International {{Conference}} on {{Learning Representations}}}, October 2021.

\bibitem[Gu and Dao(2023)]{gu2023.MambaLinearTimeSequence}
Albert Gu and Tri Dao.
\newblock Mamba: {{Linear-Time Sequence Modeling}} with {{Selective State Spaces}}, December 2023.

\bibitem[De et~al.(2024)De, Smith, Fernando, Botev, {Cristian-Muraru}, Gu, Haroun, Berrada, Chen, Srinivasan, Desjardins, Doucet, Budden, Teh, Pascanu, De~Freitas, and Gulcehre]{de2024.GriffinMixingGated}
Soham De, Samuel~L. Smith, Anushan Fernando, Aleksandar Botev, George {Cristian-Muraru}, Albert Gu, Ruba Haroun, Leonard Berrada, Yutian Chen, Srivatsan Srinivasan, Guillaume Desjardins, Arnaud Doucet, David Budden, Yee~Whye Teh, Razvan Pascanu, Nando De~Freitas, and Caglar Gulcehre.
\newblock Griffin: {{Mixing Gated Linear Recurrences}} with {{Local Attention}} for {{Efficient Language Models}}, February 2024.

\bibitem[Katharopoulos et~al.(2020)Katharopoulos, Vyas, Pappas, and Fleuret]{katharopoulos2020.TransformersAreRNNs}
Angelos Katharopoulos, Apoorv Vyas, Nikolaos Pappas, and Fran{\c c}ois Fleuret.
\newblock Transformers are {{RNNs}}: {{Fast Autoregressive Transformers}} with {{Linear Attention}}, August 2020.

\bibitem[Sun et~al.(2023)Sun, Dong, Huang, Ma, Xia, Xue, Wang, and Wei]{sun2023retentive}
Yutao Sun, Li~Dong, Shaohan Huang, Shuming Ma, Yuqing Xia, Jilong Xue, Jianyong Wang, and Furu Wei.
\newblock Retentive network: {{A}} successor to transformer for large language models.
\newblock \emph{arXiv preprint arXiv:2307.08621}, 2023.

\bibitem[Yang et~al.(2023)Yang, Wang, Shen, Panda, and Kim]{yang2023.GatedLinearAttention}
Songlin Yang, Bailin Wang, Yikang Shen, Rameswar Panda, and Yoon Kim.
\newblock Gated {{Linear Attention Transformers}} with {{Hardware-Efficient Training}}, December 2023.

\bibitem[Vaswani et~al.(2017)Vaswani, Shazeer, Parmar, Uszkoreit, Jones, Gomez, Kaiser, and Polosukhin]{vaswani2017.AttentionAllYou}
Ashish Vaswani, Noam Shazeer, Niki Parmar, Jakob Uszkoreit, Llion Jones, Aidan~N Gomez, {\L}ukasz Kaiser, and Illia Polosukhin.
\newblock Attention is {{All}} you {{Need}}.
\newblock In \emph{Advances in {{Neural Information Processing Systems}}}, volume~30. Curran Associates, Inc., 2017.

\bibitem[Press et~al.(2022)Press, Smith, and Lewis]{press2022.TrainShortTest}
Ofir Press, Noah~A. Smith, and Mike Lewis.
\newblock Train {{Short}}, {{Test Long}}: {{Attention}} with {{Linear Biases Enables Input Length Extrapolation}}, April 2022.

\bibitem[Su et~al.(2022)Su, Lu, Pan, Murtadha, Wen, and Liu]{su2022.RoFormerEnhancedTransformer}
Jianlin Su, Yu~Lu, Shengfeng Pan, Ahmed Murtadha, Bo~Wen, and Yunfeng Liu.
\newblock {{RoFormer}}: {{Enhanced Transformer}} with {{Rotary Position Embedding}}, August 2022.

\bibitem[Sun et~al.(2022)Sun, Dong, Patra, Ma, Huang, Benhaim, Chaudhary, Song, and Wei]{sun2022.LengthExtrapolatableTransformer}
Yutao Sun, Li~Dong, Barun Patra, Shuming Ma, Shaohan Huang, Alon Benhaim, Vishrav Chaudhary, Xia Song, and Furu Wei.
\newblock A {{Length-Extrapolatable Transformer}}, December 2022.

\bibitem[Chen et~al.(2023{\natexlab{a}})Chen, Li, Meng, Liang, and Bing]{chen2023.CLEXContinuousLength}
Guanzheng Chen, Xin Li, Zaiqiao Meng, Shangsong Liang, and Lidong Bing.
\newblock {{CLEX}}: {{Continuous Length Extrapolation}} for {{Large Language Models}}, October 2023{\natexlab{a}}.

\bibitem[Martin and Cundy(2018)]{martin2018.ParallelizingLinearRecurrent}
Eric Martin and Chris Cundy.
\newblock Parallelizing {{Linear Recurrent Neural Nets Over Sequence Length}}.
\newblock In \emph{International {{Conference}} on {{Learning Representations}}}, February 2018.

\bibitem[Yuan et~al.(2022)Yuan, Coenen, Reif, and Ippolito]{yuan2022.WordcraftStoryWriting}
Ann Yuan, Andy Coenen, Emily Reif, and Daphne Ippolito.
\newblock Wordcraft: {{Story Writing With Large Language Models}}.
\newblock In \emph{27th {{International Conference}} on {{Intelligent User Interfaces}}}, {{IUI}} '22, pages 841--852, New York, NY, USA, March 2022. Association for Computing Machinery.
\newblock ISBN 978-1-4503-9144-3.
\newblock \doi{10.1145/3490099.3511105}.

\bibitem[Chen et~al.(2023{\natexlab{b}})Chen, Wu, Chitta, Jaeger, Geiger, and Li]{chen2023.EndtoendAutonomousDriving}
Li~Chen, Penghao Wu, Kashyap Chitta, Bernhard Jaeger, Andreas Geiger, and Hongyang Li.
\newblock End-to-end {{Autonomous Driving}}: {{Challenges}} and {{Frontiers}}, June 2023{\natexlab{b}}.

\bibitem[Marschall et~al.(2020)Marschall, Cho, and Savin]{marschall2020.UnifiedFrameworkOnline}
Owen Marschall, Kyunghyun Cho, and Cristina Savin.
\newblock A unified framework of online learning algorithms for training recurrent neural networks.
\newblock \emph{The Journal of Machine Learning Research}, 21\penalty0 (1):\penalty0 5320--5353, 2020.

\bibitem[Gao et~al.(2020)Gao, Biderman, Black, Golding, Hoppe, Foster, Phang, He, Thite, Nabeshima, Presser, and Leahy]{gao2020.Pile800GBDataset}
Leo Gao, Stella Biderman, Sid Black, Laurence Golding, Travis Hoppe, Charles Foster, Jason Phang, Horace He, Anish Thite, Noa Nabeshima, Shawn Presser, and Connor Leahy.
\newblock The {{Pile}}: {{An 800GB Dataset}} of {{Diverse Text}} for {{Language Modeling}}, December 2020.

\bibitem[Wang and Xue(2023)]{wang2023.StatespaceModelsLayerwisea}
Shida Wang and Beichen Xue.
\newblock State-space models with layer-wise nonlinearity are universal approximators with exponential decaying memory.
\newblock In \emph{Thirty-Seventh {{Conference}} on {{Neural Information Processing Systems}}}, November 2023.

\bibitem[Li et~al.(2020)Li, Han, E, and Li]{li2020.CurseMemoryRecurrent}
Zhong Li, Jiequn Han, Weinan E, and Qianxiao Li.
\newblock On the {{Curse}} of {{Memory}} in {{Recurrent Neural Networks}}: {{Approximation}} and {{Optimization Analysis}}.
\newblock In \emph{International {{Conference}} on {{Learning Representations}}}, October 2020.

\bibitem[Jiang et~al.(2023)Jiang, Li, Li, and Wang]{jiang2023.BriefSurveyApproximationa}
Haotian Jiang, Qianxiao Li, Zhong Li, and Shida Wang.
\newblock A {{Brief Survey}} on the {{Approximation Theory}} for {{Sequence Modelling}}.
\newblock \emph{Journal of Machine Learning}, 2\penalty0 (1):\penalty0 1--30, June 2023.
\newblock ISSN 2790-203X, 2790-2048.
\newblock \doi{10.4208/jml.221221}.

\bibitem[Jaeger(2002)]{jaeger2002.TutorialTrainingRecurrent}
Herbert Jaeger.
\newblock Tutorial on training recurrent neural networks, covering {{BPPT}}, {{RTRL}}, {{EKF}} and the echo state network approach.
\newblock 2002.

\bibitem[Smith et~al.(2023)Smith, Warrington, and Linderman]{smith2023.SimplifiedStateSpace}
Jimmy T.~H. Smith, Andrew Warrington, and Scott Linderman.
\newblock Simplified {{State Space Layers}} for {{Sequence Modeling}}.
\newblock In \emph{International {{Conference}} on {{Learning Representations}}}, February 2023.

\bibitem[Rumelhart et~al.(1986)Rumelhart, Hinton, and Williams]{rumelhart1986.LearningRepresentationsBackpropagating}
David~E. Rumelhart, Geoffrey~E. Hinton, and Ronald~J. Williams.
\newblock Learning representations by back-propagating errors.
\newblock \emph{Nature}, 323\penalty0 (6088):\penalty0 533--536, October 1986.
\newblock ISSN 1476-4687.
\newblock \doi{10.1038/323533a0}.

\bibitem[Hochreiter and Schmidhuber(1997)]{hochreiter1997.LongShorttermMemory}
Sepp Hochreiter and J{\"u}rgen Schmidhuber.
\newblock Long {{Short-term Memory}}.
\newblock \emph{Neural computation}, 9:\penalty0 1735--80, December 1997.
\newblock \doi{10.1162/neco.1997.9.8.1735}.

\bibitem[Cho et~al.(2014)Cho, {van Merrienboer}, Gulcehre, Bahdanau, Bougares, Schwenk, and Bengio]{cho2014.LearningPhraseRepresentations}
Kyunghyun Cho, Bart {van Merrienboer}, Caglar Gulcehre, Dzmitry Bahdanau, Fethi Bougares, Holger Schwenk, and Yoshua Bengio.
\newblock Learning {{Phrase Representations}} using {{RNN Encoder-Decoder}} for {{Statistical Machine Translation}}.
\newblock \emph{Proceedings of the 2014 Conference on Empirical Methods in Natural Language Processing}, pages 1724--1734, September 2014.

\bibitem[Bengio et~al.(1994)Bengio, Simard, and Frasconi]{bengio1994.LearningLongtermDependencies}
Y.~Bengio, P.~Simard, and P.~Frasconi.
\newblock Learning long-term dependencies with gradient descent is difficult.
\newblock \emph{IEEE Transactions on Neural Networks}, 5\penalty0 (2):\penalty0 157--166, March 1994.
\newblock ISSN 1941-0093.
\newblock \doi{10.1109/72.279181}.

\bibitem[Hochreiter(1998)]{hochreiter1998.VanishingGradientProblem}
Sepp Hochreiter.
\newblock The {{Vanishing Gradient Problem During Learning Recurrent Neural Nets}} and {{Problem Solutions}}.
\newblock \emph{International Journal of Uncertainty, Fuzziness and Knowledge-Based Systems}, 06\penalty0 (02):\penalty0 107--116, April 1998.
\newblock ISSN 0218-4885, 1793-6411.
\newblock \doi{10.1142/S0218488598000094}.

\bibitem[Wang et~al.(2023)Wang, Li, and Li]{wang2023.InverseApproximationTheory}
Shida Wang, Zhong Li, and Qianxiao Li.
\newblock Inverse {{Approximation Theory}} for {{Nonlinear Recurrent Neural Networks}}.
\newblock In \emph{The {{Twelfth International Conference}} on {{Learning Representations}}}, October 2023.

\bibitem[Wang and Li(2023)]{wang2023.StableSSMAlleviatingCurse}
Shida Wang and Qianxiao Li.
\newblock {{StableSSM}}: {{Alleviating}} the {{Curse}} of {{Memory}} in {{State-space Models}} through {{Stable Reparameterization}}, November 2023.

\bibitem[Chen et~al.(2023{\natexlab{c}})Chen, Wong, Chen, and Tian]{chen2023.ExtendingContextWindow}
Shouyuan Chen, Sherman Wong, Liangjian Chen, and Yuandong Tian.
\newblock Extending {{Context Window}} of {{Large Language Models}} via {{Positional Interpolation}}, June 2023{\natexlab{c}}.

\bibitem[Chen et~al.(2019)Chen, Rubanova, Bettencourt, and Duvenaud]{chen2019.NeuralOrdinaryDifferential}
Ricky T.~Q. Chen, Yulia Rubanova, Jesse Bettencourt, and David Duvenaud.
\newblock Neural {{Ordinary Differential Equations}}, December 2019.

\bibitem[Raffel et~al.(2020)Raffel, Shazeer, Roberts, Lee, Narang, Matena, Zhou, Li, and Liu]{raffel2020.ExploringLimitsTransfer}
Colin Raffel, Noam Shazeer, Adam Roberts, Katherine Lee, Sharan Narang, Michael Matena, Yanqi Zhou, Wei Li, and Peter~J. Liu.
\newblock Exploring the {{Limits}} of {{Transfer Learning}} with a {{Unified Text-to-Text Transformer}}.
\newblock \emph{Journal of Machine Learning Research}, 21\penalty0 (140):\penalty0 1--67, 2020.
\newblock ISSN 1533-7928.

\bibitem[{Al-Khateeb} et~al.(2023){Al-Khateeb}, Dey, Soboleva, and Hestness]{al-khateeb2023.PositionInterpolationImproves}
Faisal {Al-Khateeb}, Nolan Dey, Daria Soboleva, and Joel Hestness.
\newblock Position {{Interpolation Improves ALiBi Extrapolation}}, October 2023.

\bibitem[Demanet and Townsend(2019)]{demanet2019.StableExtrapolationAnalytic}
Laurent Demanet and Alex Townsend.
\newblock Stable {{Extrapolation}} of {{Analytic Functions}}.
\newblock \emph{Foundations of Computational Mathematics}, 19\penalty0 (2):\penalty0 297--331, April 2019.
\newblock ISSN 1615-3375, 1615-3383.
\newblock \doi{10.1007/s10208-018-9384-1}.

\bibitem[Cover and Thomas(2006)]{cover2006.ElementsInformationTheory}
Thomas~M. Cover and Joy~A. Thomas.
\newblock \emph{Elements of {{Information Theory}} ({{Wiley Series}} in {{Telecommunications}} and {{Signal Processing}})}.
\newblock Wiley-Interscience, USA, June 2006.
\newblock ISBN 978-0-471-24195-9.

\end{thebibliography}

\newpage
\appendix

\section{Comparison of state-space models and nonlinear recurrent neural networks}
\label{sec:comparison_ssm_rnn}

Here we give the formulation of single-layer recurrent neural networks (RNNs)~\citep{rumelhart1986.LearningRepresentationsBackpropagating}. 
In nonlinear RNNs the activation $\sigma$ is applied in the temporal direction. 
\begin{align}
    h_{k+1}   & = \bm{\sigma}(Wh_k + Ux_k + b), \quad h_0 = 0\\
    \hat{y}_k & = C h_k, \quad 1 \leq k \leq T.
\end{align}
The corresponding continuous-time form of RNNs is 
\begin{align}
    \frac{dh_t}{dt} = \bm{\sigma}(Wh_t + Ux_t + b), \quad \hat{y}_t = C h_t. 
\end{align}

\paragraph{Truncated backpropagation through time} 
Due to the nonlinear dynamics of nonlinear RNNs, backpropagation through time (BPTT)~\citep{jaeger2002.TutorialTrainingRecurrent} is the standard approach to evaluate the gradient.
Due to the vanishing/exploding gradient issue~\citep{bengio1994.LearningLongtermDependencies, hochreiter1998.VanishingGradientProblem}, truncated backpropagation through time ~\citep{jaeger2002.TutorialTrainingRecurrent} is widely used to speedup the training. 
In this paper, our previous-initialized hidden state is similar to this T-BPTT method.

\section{Theoretical backgrounds}

The definitions and theorems are well-known results from information theory. 
We collect the definition for the completeness~\citep{cover2006.ElementsInformationTheory}. 

\subsection{Entropy, conditional entropy and chain rule}

Here we let $X$ and $Y$ denote random variable. 

Entropy is a measure of the uncertainty of a random variable:
\begin{equation}
    H(X) = \mathbb{E}_p \log \left (\frac{1}{p(X)} \right ).
\end{equation}

Joint entropy:
\begin{equation}
    H(X, Y) = \mathbb{E}_p \log \left (\frac{1}{p(X, Y)} \right ).
\end{equation}

Conditional entropy:
\begin{equation}
    H(Y|X) = \mathbb{E}_p \log \left (\frac{1}{p(Y|X)} \right ).
\end{equation}

Chain rule:
\begin{equation}
    H(X, Y) = H(X) + H(Y | X).
\end{equation}

\subsection{Relative entropy, mutual information}

The \textbf{relative entropy} $D(p || q)$ is a measure of the inefficiency of assuming that the distribution is $q$ when the true distribution is $p$. 
\begin{equation}
    D(p||q) = E_p \log \left (\frac{p(X)}{q(X)} \right ).
\end{equation}

Chain rule for relative entropy:
\begin{equation}
    D(p(x, y) || q(x, y))  = D(p(x) || q(x)) + D(p(y|x) || q(y|x)). 
\end{equation}

Mutual information:
\begin{align}
    I(X; Y) 
    & = E_{p(X, Y)} \log \frac{p(X, Y)}{p(X) p(Y)} \\
    & = H(X) - H(X|Y) \\
    & = H(X) + H(Y) - H(X,Y)
\end{align}

\begin{theorem}[Information inequality]
    \begin{equation}
        D(p || q) \geq 0.
    \end{equation}
    with equality if and only if $p(x) = q(x)$ for all x.
\end{theorem}

\begin{corollary}[Nonnegativity of mutual information]
    For any two random variables, $X, Y$, 
    \begin{equation}
        I(X; Y) \geq 0. 
    \end{equation}
    This comes from the fact by taking $p$ to be $p(X, Y)$ and $q$ to be $p(X)p(Y)$. 
\end{corollary}

\begin{theorem}[Conditioning reduces entropy]
\label{thm:conditioning_reduced_entropy}
    \begin{equation}
        H(X|Y) \leq H(X). 
    \end{equation}
    with equality if and only if $X$ and $Y$ are independent. 
\end{theorem}

\subsection{Riesz representation theorem for linear functional}
\label{subsec:Riesz_representation_theorem}

\begin{theorem}[Riesz-Markov-Kakutani representation theorem]
\label{thm:riesz_representation_theorem}
    Assume $H : C_0(\mathbb{R}, \mathbb{R}^d) \mapsto \mathbb{R}$ is a linear and continuous functional. 
    Then there exists a unique, vector-valued, regular, countably additive signed measure $\mu$ on $\mathbb{R}$ such that
    \begin{align}
        H(\mathbf{x}) = \int_{\mathbb{R}} x_s^\top d\mu(s)
        = \sum_{i=1}^{d} \int_{\mathbb{R}} x_{s,i} d\mu_i(s).
    \end{align}
    In addition, we have the linear functional norm 
    \begin{equation}
    \label{eq:linear_functional_norm}
        \| H \|_{\infty} := \sup_{\| \mathbf{x} \|_\mathcal{X} \leq 1} | H(\mathbf{x}) | = \|\mu\|_1(\mathbb{R}) := \sum_i |\mu_i|(\mathbb{R}).
    \end{equation}
\end{theorem}

\section{Theoretical results and proofs}
\label{sec:theoretical_results}

In \cref{subsec:associativity}, we give the proof for input-dependent gating in state-space models is associative. 
In \cref{subsec:sensitivity_of_recurrent_weights}, we show the dependency of recurrent weights range with respect to the finite precision range and why the corresponding gradient values might be unbounded. 

\subsection{Input-dependent gating in state-space models is associative}
\label{subsec:associativity}

\begin{table*}[tbh!]
    \caption{Difference between S5, Mamba and Gated Linear Attention in terms of the recurrent weight and hidden states dimensions. 
    }
    \label{table:diffs_s5_mamba_gla}
    \centering
    \begin{tabular}{c|cc}
    \toprule
                    & $h$ is vector & $h$ is matrix \\
    \midrule
    $W$ diagonal   & S5~\citep{smith2023.SimplifiedStateSpace}, Mamba   & N/A \\
    $W$ full       & Traditional SSM & Gated Linear Attention~\citep{yang2023.GatedLinearAttention} \\
    \bottomrule
    \end{tabular}
\end{table*}

Consider the following binary operator defined for tuple element $(W, h)$ as follow:
\begin{equation}
    (W_1, h_1) \circ (W_2, h_2) = (W_2 \odot W_1, h_1 + W_1 \odot h_2).
\end{equation}
Notice that $W$ and $h$ can depend on input value $x$. 

\begin{theorem}[Associativity of binary operation in state-space models]
    \begin{equation}
        \bigg ((W_1, h_1) \circ (W_2, h_2) \bigg ) \circ (W_3, h_3) = (W_1, h_1) \circ \bigg( (W_2, h_2) \circ (W_3, h_3) \bigg)
    \end{equation}
\end{theorem}

\begin{proof}
\begin{align}
    & \bigg ((W_1, h_1) \circ (W_2, h_2)\bigg ) \circ (W_3, h_3) \\
    & = (W_2 \odot W_1, h_1 + W_1 \odot h_2) \circ (W_3, h_3) \\
    & = (W_3 \odot W_2 \odot W_1, h_1 + W_1 \odot h_2 + W_1 \odot W_2 \odot h_3) \\
    & = (W_1, h_1) \circ (W_3 \odot W_2, h_2 + W_2 \odot h_3) \\
    & = (W_1, h_1) \circ \bigg ((W_2, h_2) \circ (W_3, h_3)\bigg )
\end{align}
    
\end{proof}

\subsection{Sensitivity of recurrent weights}
\label{subsec:sensitivity_of_recurrent_weights}

Let $M$ be the maximum value in given finite precision machine. 
Let $\lambda = \max(\textrm{diag}(\Lambda)) (< 1)$ be the (largest) memory decay mode in state-space models: 
An estimate for the hidden states scale as follow:
\begin{align}
    |h_T|_{\infty}
    & = \left | h_0 + \sum_{k=1}^T \Lambda^k U x_{k-1} \right |_{\infty} \\
    & \leq |h_0|_{\infty} + \frac{1-\lambda^T}{1-\lambda} |U|_1 \sup_k|x_k|_{\infty}
\end{align}
To prevent the overflow of hidden states $|h_T|_2 \leq M$, as the sequence length increases $T \to \infty$, a \textbf{sufficient} condition for the eiganvalue ranges is
\begin{equation}
    \lambda < 1 - \frac{|U|_1 \sup_k |x_k|_{\infty}}{M - |h_0|_{\infty}}. 
\end{equation}
As the learning process of long-term memory requires the slow decay of information within hidden states, achieving long-term memory implies that the parameter $\lambda$ gets close to 1.
This result shows that if we use low-bit quantization (with small $M$), the hidden state might be unbounded by $M$ and therefore the training can be unstable due to overflow issues. 
The overflow issue is \textbf{more severe for large models} as $|U|_1$ scale up in $O(m)$ with respect to the hidden dimension $m$.

\subsection{Existence of weak length extension}
\label{subsec:existence_of_weak_length_extension}

Consider the language modeling as the learning of sequence of random variables. 
It can be seen that such language modeling should obey the weak length extension in the information theory framework. 

\begin{proposition}
    Let the dataset of autoregressive language modeling sampled from the (potentially infinite) sequence of random variables, then we have the existence of weak length extension: 
    \begin{equation}
        H(X_{k+1} | X_1, \dots, X_k) \leq H(X_{k+1} | X_2, \dots, X_k) \leq H(X_{k+1} | X_k) \leq H(X_{k+1}).
    \end{equation}
\end{proposition}

This is a direct result from the conditioning reduces entropy theorem. 
We just need to repeatedly apply the above \cref{thm:conditioning_reduced_entropy} to $X = (X_{k+1} | X_{i+1}, \dots, X_{k})$ while $Y=X_i$.

\section{Additional numerical results}
\label{sec:additional_numerical_results}

In this section, we provide additional numerical results to show previous-initialized hidden state is also effective for S5 (\cref{subsec:comparison_of_different_hidden_states_initialization}), the empirical similarity between length extension and polynomial extrapolation (\cref{{subsec:overfit_in_length_extension}}) and the benefit of previous-hidden states in training for GRU (\cref{subsec:generalization_to_other_recurrent_model}).

\subsection{Comparison of different hidden states initialization}
\label{subsec:comparison_of_different_hidden_states_initialization}

In \cref{fig:comparison_of_two_hidden_states_initailizations} we show the effect of previous-initialized hidden state help the Mamba to have length extension. 
In \cref{fig:S5_comparison_of_previous_and_zero_over_unshuffle}, we further compare the influence of hidden state initialization schemes during the training of the models. 
It can be seen the S5 trained with previous-initialized hidden states can have length extension from 16 to 32768.

\begin{figure}[ht!]
    {
    \centering
    \subfigure[][Previous-initialized S5 + unshuffled dataloader]{\includegraphics[width=0.47\textwidth]{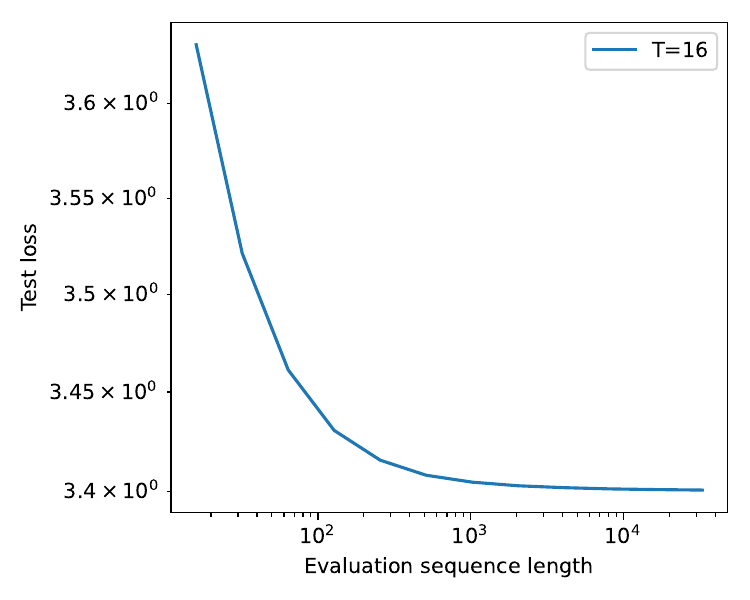}
    }
    \subfigure[][Zero-initialized S5 and unshuffled dataloader]{\includegraphics[width=0.47\textwidth]{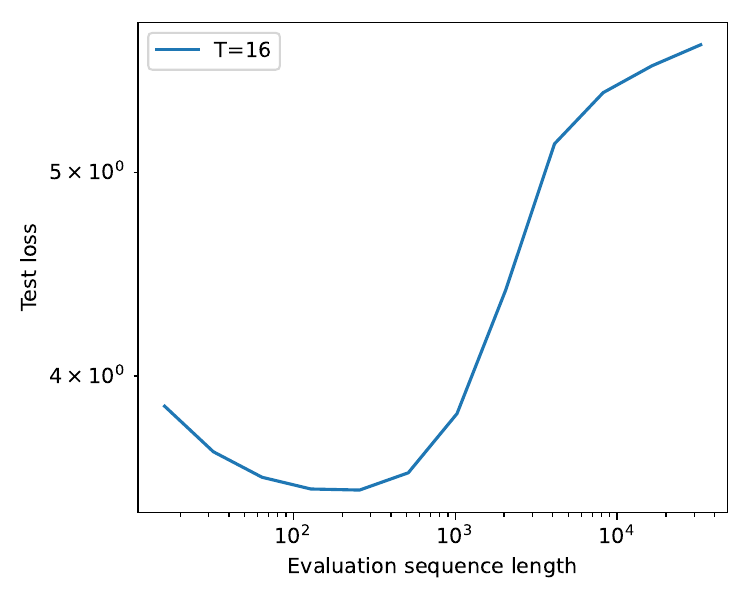}
    }

    \caption{
        The 30M S5 model trained over Wikitext103 using previous-initialized hidden states has (weak) length extension capability up to sequence length 32768 while the zero-initialized model fails to have monotone length extension capability. 
    }
    \label{fig:S5_comparison_of_previous_and_zero_over_unshuffle}
}
\end{figure}

\subsection{Overfitting in length extension}
\label{subsec:overfit_in_length_extension}

We show the length extension for nonlinear state-space models are similar to polynomial extrapolation in the following sense: 
This escalation in parameters significantly complicates the length extension process, necessitating a proportional increase in the training sequence length for zero-initialized models from 64 and 256 to a substantial 1024. 
This correspond to the overfitting argument as the larger model size will require larger sequence length (more data for the evaluation of the memory function $\rho$.)
It is anticipated that, for large models trained with zero-initialized hidden states, one cannot use length $T=2048$ to train a model with length extension capability. 
The missing last row in 370M is due to the out-of-memory issue. 
Two missing values come from overflow issue. 

\begin{figure}[t!]{
    \centering
    \subfigure[][12M]{\includegraphics[width=0.47\textwidth]{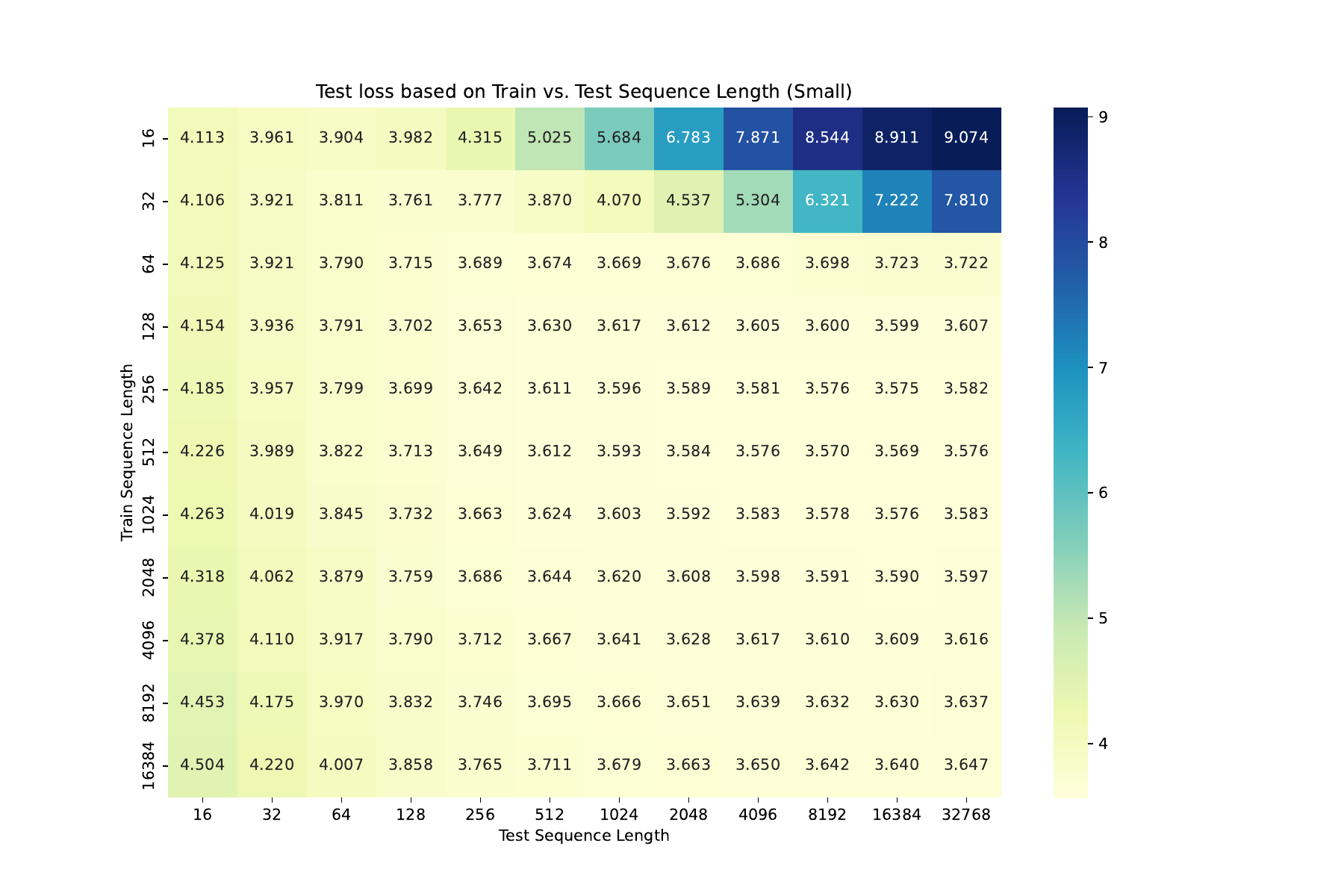}}
    \subfigure[][37M]{\includegraphics[width=0.47\textwidth]{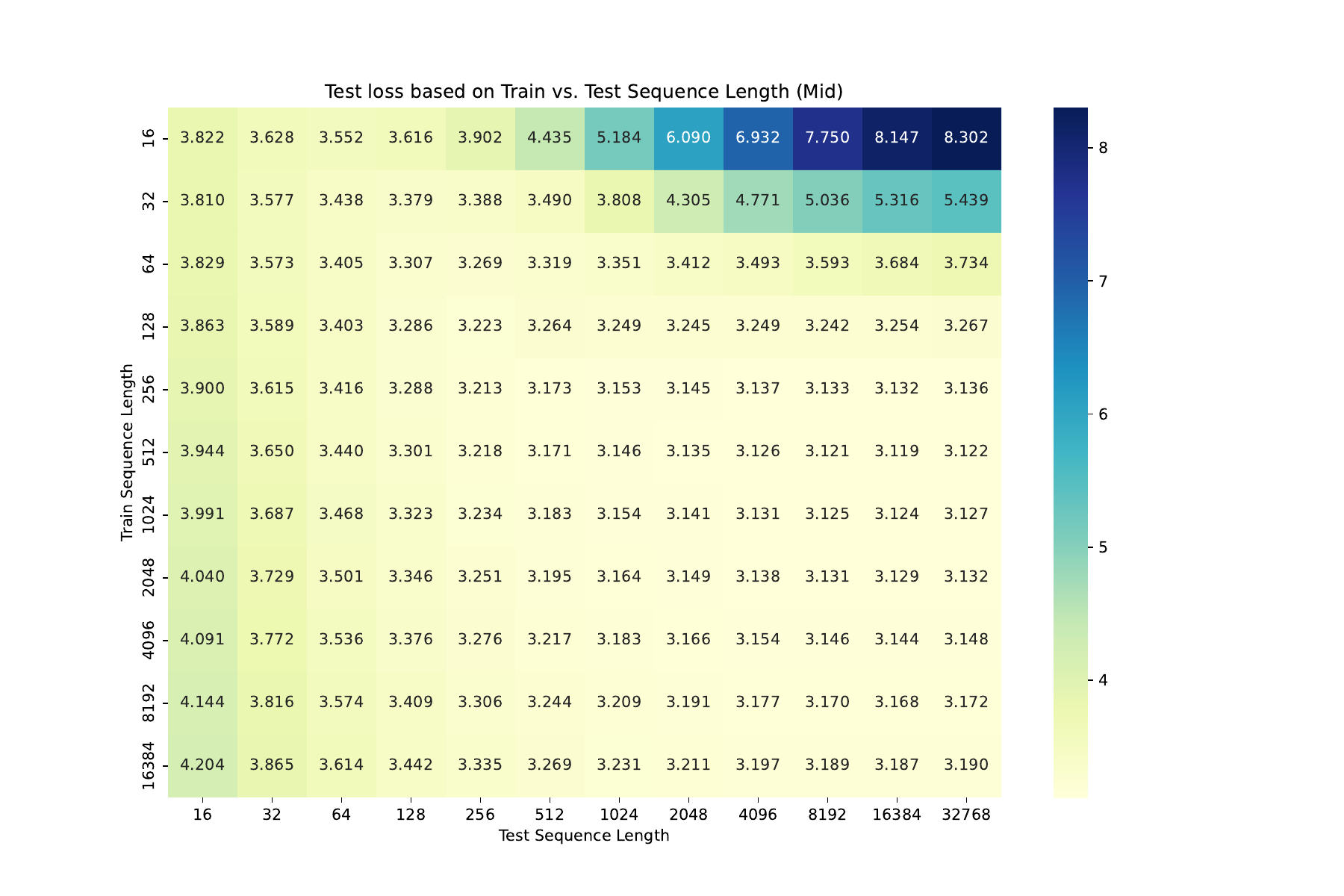}}
    
    \subfigure[][127M]{\includegraphics[width=0.47\textwidth]{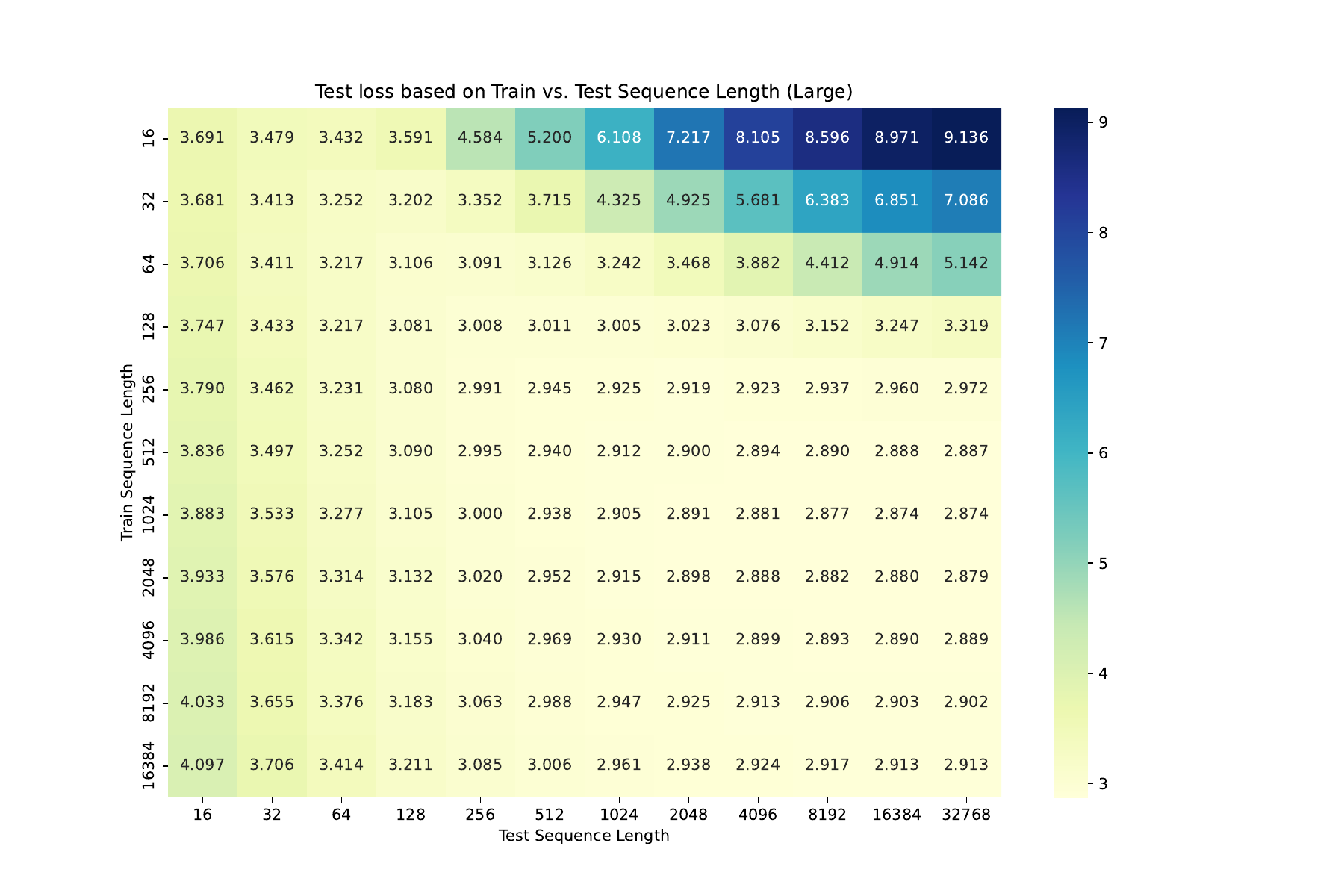}}
    \subfigure[][370M]{\includegraphics[width=0.47\textwidth]{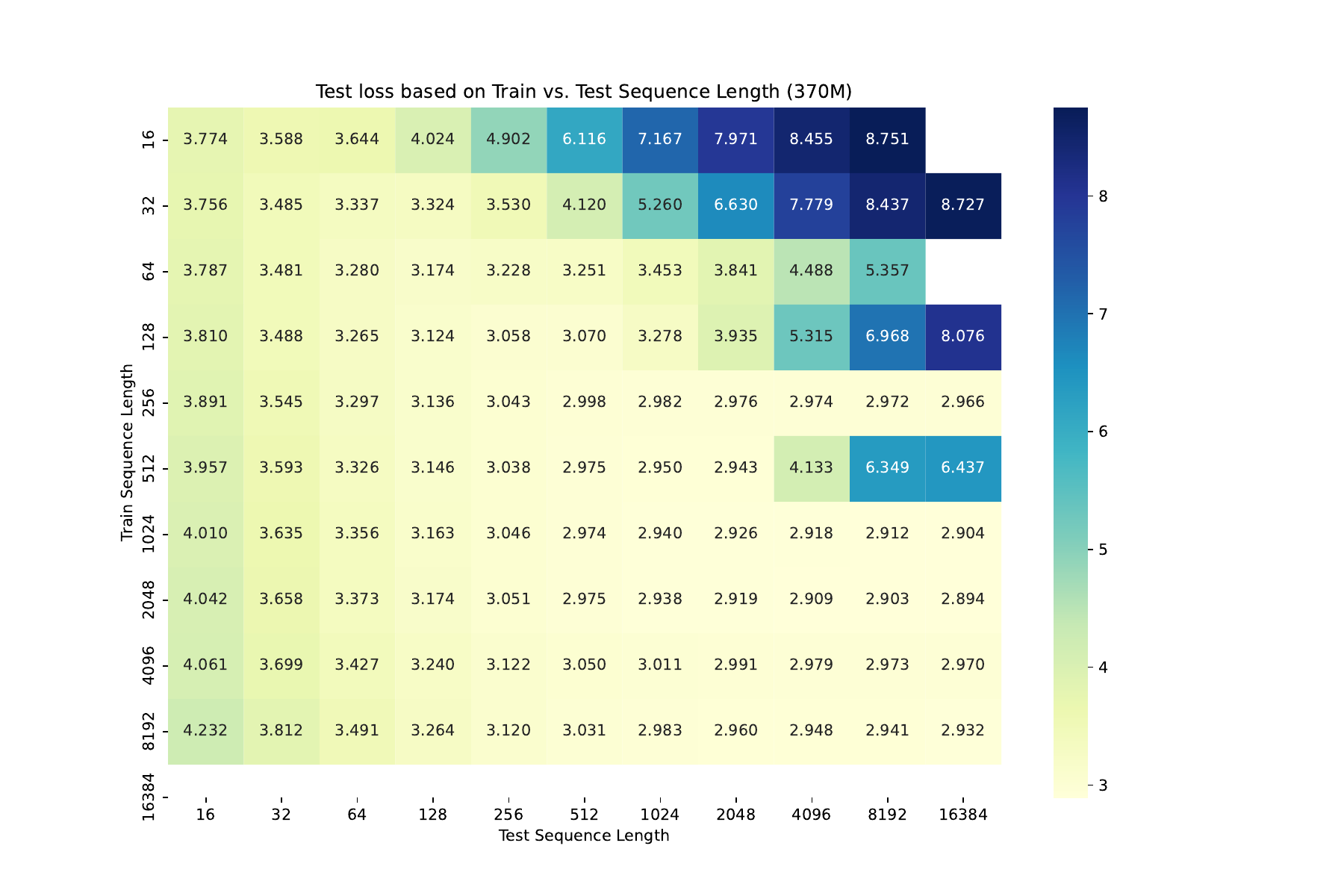}}

    \caption{
        Various sizes of zero-initialized S5 models were trained on Wikitext103. 
        The performance observed in the upper triangular area of the graph illustrates their capacity for length extension. 
        To maintain consistency, we scaled up the models while keeping the training settings constant. 
        Notably, as models are initialized with zero-hidden states, larger models necessitate longer training contexts to effectively handle length extension
    }
    \label{figs:S5_zero_length_extension_larger_worse}
}

\end{figure}

\subsection{Generalization to other recurrent model}
\label{subsec:generalization_to_other_recurrent_model}

In addition to the state-space models, we extend our comparison to different hidden state initialization schemes for a 6-layer GRU with 30 million parameters, as depicted in \cref{fig:GRU_Wikitext_training_benefit}. 
The results demonstrate that initializing with previous hidden states can enhance the training performance of the GRU model.

\begin{figure*}[ht!]{
    \centering
    \subfigure[][$T=16$]{\includegraphics[width=0.47\textwidth]{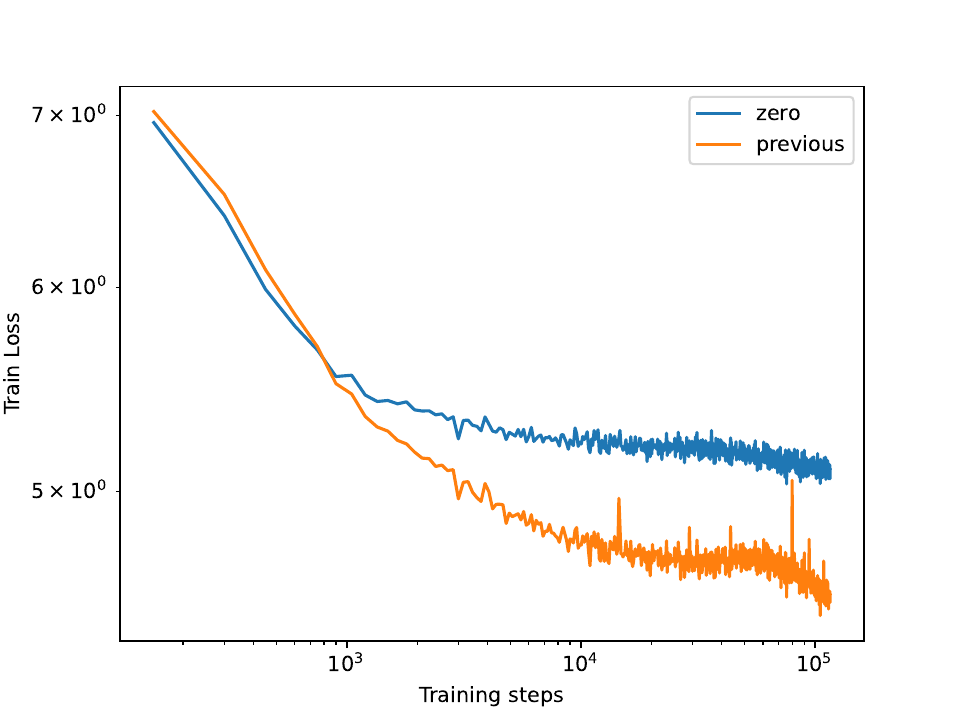}}
    \subfigure[][$T=32$]{\includegraphics[width=0.47\textwidth]{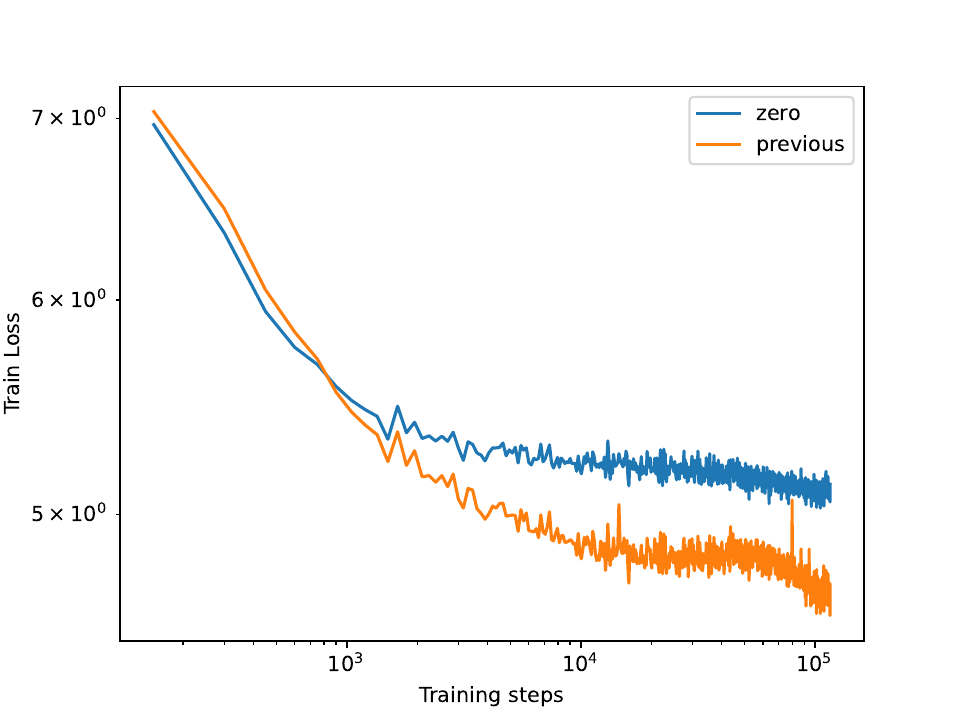}}

    \subfigure[][$T=64$]{\includegraphics[width=0.47\textwidth]{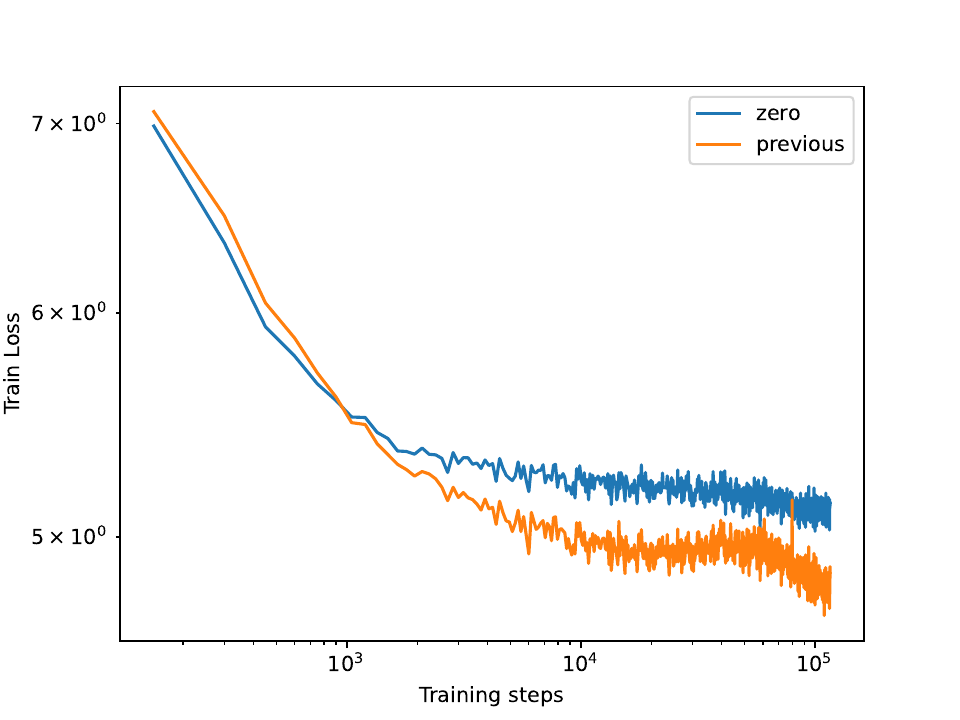}}
    \subfigure[][$T=128$]{\includegraphics[width=0.47\textwidth]{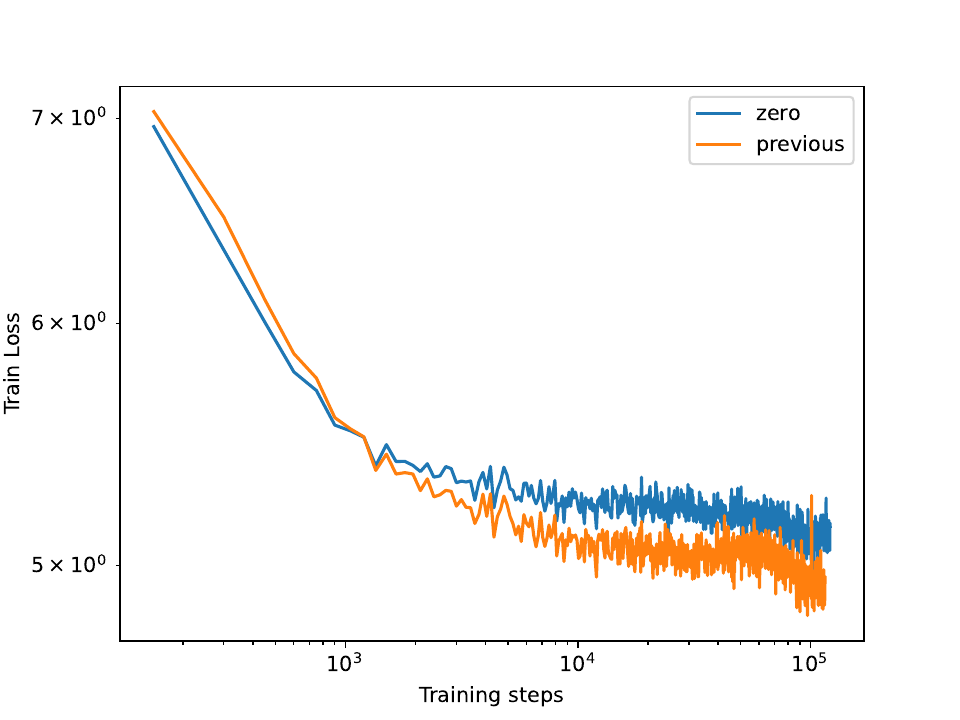}}
    \caption{
        Training with previous hidden states initialization on Wikitext103 enhances GRU training performance. 
        Here $T$ is the training sequence length. 
    }
    \label{fig:GRU_Wikitext_training_benefit}
}
\end{figure*}

\end{document}